\documentclass[11pt,a4paper]{article}

\usepackage[utf8]{inputenc}
\usepackage[T1]{fontenc}
\usepackage[margin=1in]{geometry}
\usepackage{amsmath,amssymb,amsfonts,amsthm,mathtools}
\usepackage{booktabs}
\usepackage{microtype}
\usepackage{natbib}
\usepackage{enumitem}
\usepackage{float}
\usepackage{hyperref}

\hypersetup{
    colorlinks=true,
    linkcolor=blue,
    filecolor=magenta,
    urlcolor=cyan,
    citecolor=blue,
}

\theoremstyle{plain}
\newtheorem{theorem}{Theorem}
\newtheorem{lemma}{Lemma}
\newtheorem{corollary}{Corollary}

\theoremstyle{definition}

\newtheorem{assumption}{Assumption}

\DeclareMathOperator{\Prb}{P}
\DeclareMathOperator{\E}{E}

\title{\textbf{The Need for an External Observer\\
Formalizing the Sufficiency Gap:\\
A Mathematical Extension of Mixture Identifiability and Contextual Grounding in Sequence Models}}
\author{\textbf{Francesco Corielli}\thanks{This work serves as an extension and companion framework to Corielli (2026).}}
\date{\today}

\begin{document}

\maketitle

\begin{abstract}
This paper provides a formal companion model to \citet{corielli2026}. It  gives a minimal binary mixture model in which several familiar limitations of next-token prediction can be derived in closed form. The model shows that even an ideal infinite-capacity predictor of the text-only marginal law can become overconfident when the observed prefix is compatible with the wrong latent regime. The resulting entropy difference is not an ordinary optimization error; it is a sufficiency gap caused by marginalization over an unobserved state. We then formalize retrieval, tool use, and external grounding through an auxiliary binary signal with fidelity $\gamma \in [1/2,1]$. The resulting Bayesian update yields a contextual dominance threshold: a corrective signal reverses the posterior odds induced by the textual history exactly when its fidelity exceeds the text-only posterior weight assigned to the misleading regime. This threshold reduces, but does not generally eliminate, the sufficiency gap; complete closure requires perfect revelation of the relevant latent state or an equivalent verification mechanism. The contribution is therefore a formal clarification and synthesis: external grounding helps not because it magically adds intelligence to next-token prediction, but because it can reduce residual dependence on latent circumstances when the signal is informative and the model has learned how to use it.
\end{abstract}

\section{Introduction}

A common description of language modeling is that models estimate or approximate a conditional distribution over the next token given the observed textual prefix, a view rooted in statistical language modeling and information-theoretic accounts of language prediction \citet{shannon1948,shannon1951,manning1999,rosenfeld2000}. This description is statistically meaningful only under additional assumptions. In neural language modeling, including early neural probabilistic models and later transformer-based systems, a training corpus consists of realized token trajectories, not directly observed conditional laws \citet{bengio2003,vaswani2017,radford2019,brown2020}. Moreover, human language production is conditioned not only on previous words, but also on facts, intentions, goals, institutions, beliefs, physical circumstances, and task-specific constraints; this point is closely related to the distinction between form and grounded meaning emphasized by \citet{bender2020}. \citet{corielli2026} distinguishes three objects that are often conflated: the full conditional language process, the marginal text-only law obtained by integrating out latent circumstances, and the model-induced distribution learned from finite corpora.

This paper translates that distinction into a compact analytical example. We construct a binary data-generating process composed of two regimes. In one regime, the observable text sequence is deterministic and strictly alternating. In the other, it is an independent fair coin. The latent regime is not directly observed. An ideal text-only predictor therefore infers a posterior probability over regimes from the observed prefix and predicts from the resulting mixture conditional. This model can be statistically optimal relative to the text-only marginal distribution and still be epistemically misleading in the actual latent state.

The central object of analysis is a \emph{sufficiency gap}: the gap between the entropy of the true continuation conditional on the relevant latent state and the entropy of the text-only marginal prediction. In the broader language of \citet{corielli2026}, such gaps become operationally dangerous in heterogeneous, historically sampled, and only partially ergodic archives. Nevertheless, the mathematical point in the present paper is sharper: the failure can arise even after perfect recovery of the marginal text-only law. The problem is not merely insufficient model capacity or poor optimization. It is the loss of latent-state information under marginalization.

The paper then introduces an auxiliary grounding signal. This signal represents retrieval-augmented generation, tool use, programmatic checks, human review, or any external observer that supplies evidence about the omitted latent state. Its fidelity is parameterized by $\gamma \in [1/2,1]$. The resulting posterior update gives a closed-form contextual dominance threshold. If the grounding signal is too noisy relative to the text-only posterior induced by the misleading prefix, the archive-induced structural bias remains dominant. If the signal is sufficiently precise, it reverses posterior odds and reduces the sufficiency gap. Full gap closure, however, requires perfect or effectively perfect access to the latent state.

The contribution is deliberately architecture-neutral. The analysis does not depend on transformers, parameter counts, reinforcement learning, or any particular decoding system. It applies to any predictor trained on realized sequences and evaluated through a text-only or augmented conditional distribution.

\subsection*{Relation to existing literature and contribution}

Many ingredients of this argument have precedents. Prior work has examined the distinction between linguistic form and grounding \citet{bender2020}; hallucination, factuality, and model uncertainty \citet{ji2023,kadavath2022}; retrieval-augmented language modeling \citet{lewis2020,borgeaud2022}; tool use and neuro-symbolic orchestration \citet{karpas2022,schick2023,yao2023,gao2023}; in-context learning and Bayesian or latent-variable interpretations of prompting \citet{brown2020,min2022,xie2022}; instruction tuning and alignment interfaces \citet{ouyang2022}; and recursive training on generated data or model collapse \citet{shumailov2023,shumailov2024,alemohammad2024}. 

\citet{corielli2026}, obviously, does not purport discovery of these phenomena. What \citet{corielli2026} stresses, is that they can be better understood as manifestations of a common statistical structure: prediction from a learned text-only marginal law is epistemically useful only when the observed and augmented context is sufficiently informative about the latent circumstances relevant to continuation.

The binary model presented here is therefore diagnostic rather than exhaustive. It is not intended to represent the full complexity of modern LLMs, which involve finite data, finite capacity, optimization error, alignment tuning, distribution shift, decoding policies, retrieval failures, and user-interface effects. Its purpose is to isolate one mechanism that remains even after all ordinary estimation error is removed: a predictor can be correct relative to the text-only marginal law and still be misleading relative to the realized latent state.

\section{Background: Three Distinct Distributions}

Let $X_t$ denote the token at position $t$, and let
\[
X_{\le t}=(X_1,\ldots,X_t)
\]
be the observed textual history. Let $Z_t$ denote the non-textual circumstances relevant to producing or evaluating the next token. These may include the world state, facts, speaker beliefs, goals, intentions, task constraints, institutional context, and other latent variables.

Following \citet{corielli2026}, it is useful to distinguish three distributions.

\begin{enumerate}[label=(\roman*)]
    \item The \textbf{full conditional process}
    \[
    p_{\mathrm{full}}(x_{t+1}\mid x_{\le t},z_t),
    \]
    which conditions on both text and latent circumstances.

    \item The \textbf{marginal text-only law}
    \[
    p_{\mathrm{marg}}(x_{t+1}\mid x_{\le t})
    =\int p_{\mathrm{full}}(x_{t+1}\mid x_{\le t},z)\,p(z\mid x_{\le t})\,dz,
    \]
    which averages over latent circumstances conditional on text.

    \item The \textbf{model-induced predictive distribution}
    \[
    p_\theta(x_{t+1}\mid x_{\le t}),
    \]
    which is the distribution actually represented by a trained model.
\end{enumerate}

The present paper assumes an idealized limit in which the model-induced distribution exactly recovers the relevant text-only marginal law, abstracting from finite-sample estimation and optimization issues familiar from statistical learning theory \citet{bishop2006}. This removes finite-data, architecture, and optimization failures from the analysis and isolates the conditional-independence issue rather than ordinary estimation error. If a gap remains under this idealization, it is structural rather than merely empirical.

The usefulness condition is local sufficiency. The textual prefix is sufficient for continuation when
\begin{equation}
X_{t+1}\perp Z_t\mid X_{\le t},
\end{equation}
or approximately when
\begin{equation}
I(X_{t+1};Z_t\mid X_{\le t})\approx 0.
\end{equation}
If this condition fails, a correct text-only marginal conditional may remain a poor guide to the actual situation-specific conditional law.

\section{A Binary Mixed-Regime Process}

We now define a minimal process that makes the sufficiency failure explicit. Let the latent regime process $\{Z_t\}_{t\ge 1}$ take values in $\{0,1\}$. We index the latent variable as $Z_{t+1}$ when it governs the emission of $X_{t+1}$. This is notationally equivalent to using $Z_t$ as the continuation-relevant latent state after a one-step reindexing.

Let $\mathcal V=\{0,1\}$ be the vocabulary and let $\mathcal H_t=\sigma(X_1,\ldots,X_t)$ denote the textual filtration. Assume the latent process is a homogeneous Markov chain with symmetric retention probability $\rho\in(1/2,1)$, using standard stochastic-process notation \citet{doob1953}:
\begin{equation}
\Prb(Z_{t+1}=i\mid Z_t=i)=\rho,\qquad i\in\{0,1\}.
\end{equation}
The exact value of $\rho$ is not central to the entropy calculations below; it provides a simple source of regime persistence. The broader interpretation of empirical corpus frequencies as stable process information relies on ergodic-type assumptions of the kind classically formalized by \citet{birkhoff1931}.

The full conditional law is defined as follows.

\begin{description}[leftmargin=1.2cm]
    \item[Regime 0: deterministic textual regime.] The next token alternates with certainty:
    \begin{equation}
    \Prb(X_{t+1}=x\mid \mathcal H_t,Z_{t+1}=0)
    =\mathbf 1_{\{x=1-X_t\}}.
    \end{equation}

    \item[Regime 1: random latent regime.] The next token is an independent fair coin:
    \begin{equation}
    \Prb(X_{t+1}=x\mid \mathcal H_t,Z_{t+1}=1)=\frac12,
    \qquad x\in\{0,1\}.
    \end{equation}
\end{description}

The ideal text-only predictor recovers the marginal conditional law by integrating over the posterior distribution of the latent regime:
\begin{align}
 p_{\mathrm{marg}}(X_{t+1}=x\mid\mathcal H_t)
 &=\sum_{k\in\{0,1\}}
 \Prb(Z_{t+1}=k\mid\mathcal H_t)
 \Prb(X_{t+1}=x\mid\mathcal H_t,Z_{t+1}=k).
\end{align}
Define
\begin{equation}
\pi_{t,0}:=\Prb(Z_{t+1}=0\mid\mathcal H_t).
\end{equation}
For the alternating continuation $x=1-X_t$, the marginal law becomes
\begin{equation}
\label{eq:p-marg-alt}
 p_{\mathrm{marg}}(X_{t+1}=1-X_t\mid\mathcal H_t)
 =\pi_{t,0}+\frac12(1-
\pi_{t,0})=\frac12(1+\pi_{t,0}).
\end{equation}

\begin{assumption}[Full support of the misleading regime]
\label{ass:full-support}
The prior over regimes has full support and the observed alternating history has positive likelihood under Regime 0. Hence $\pi_{t,0}>0$ whenever the observed history is compatible with Regime 0.
\end{assumption}

This assumption avoids treating the deterministic regime as impossible after observing an alternating prefix. It is the minimal support condition needed for the text-only posterior to place positive weight on the misleading structural explanation.

\section{Entropy of the Text-Only Mixture}

Let $H_2:[0,1]\to[0,1]$ denote binary entropy in bits \citet{cover2006}:
\begin{equation}
H_2(p)=-p\log_2 p-(1-p)\log_2(1-p).
\end{equation}

\begin{lemma}[Entropy of the text-only mixture]
\label{lem:entropy-bounds}
Let $\alpha=(1+\pi)/2$ for $\pi\in[0,1]$. Then $H_2(\alpha)=1$ if and only if $\pi=0$. Moreover, $H_2(\alpha)$ is strictly decreasing in $\pi$ on $(0,1)$.
\end{lemma}

\begin{proof}
By the chain rule,
\begin{equation}
\frac{d}{d\pi}H_2\!\left(\frac{1+\pi}{2}\right)
=\frac12\log_2\left(\frac{1-\alpha}{\alpha}\right)
=\frac12\log_2\left(\frac{1-\pi}{1+\pi}\right).
\end{equation}
For $\pi\in(0,1)$, the ratio $(1-\pi)/(1+\pi)$ lies in $(0,1)$, so the derivative is strictly negative. At $\pi=0$, $\alpha=1/2$ and $H_2(1/2)=1$. For $\pi>0$, strict monotonicity gives $H_2(\alpha)<1$.
\end{proof}

\section{Theorem 1: The Sufficiency Gap and False Authority}

\begin{theorem}[Sufficiency gap under latent-regime mismatch]
\label{thm:sufficiency-gap}
Suppose the true continuation regime is Regime 1, so that $Z_{t+1}=1$, but the realized textual history $\mathcal H_t$ is compatible with a deterministic alternating sequence and satisfies Assumption~\ref{ass:full-support}. Then the predictive entropy of the ideal text-only marginal model is strictly lower than the true conditional entropy:
\begin{equation}
\Delta H
:=H(p_{\mathrm{full}}\mid\mathcal H_t,Z_{t+1}=1)
-H(p_{\mathrm{marg}}\mid\mathcal H_t)>0.
\end{equation}
\end{theorem}

\begin{proof}
Under Regime 1, the true conditional distribution is uniform over the binary vocabulary. Hence
\begin{equation}
H(p_{\mathrm{full}}\mid\mathcal H_t,Z_{t+1}=1)=1\quad\text{bit}.
\end{equation}
By Eq.~\eqref{eq:p-marg-alt}, the text-only marginal predictor assigns probability
\begin{equation}
\alpha=\frac12(1+\pi_{t,0})
\end{equation}
to the alternating continuation and probability $1-\alpha$ to the non-alternating continuation. By Assumption~\ref{ass:full-support}, $\pi_{t,0}>0$. Lemma~\ref{lem:entropy-bounds} therefore implies
\begin{equation}
H(p_{\mathrm{marg}}\mid\mathcal H_t)=H_2(\alpha)<1.
\end{equation}
Consequently,
\begin{equation}
\Delta H=1-H_2(\alpha)>0.
\end{equation}
\end{proof}

The theorem proves that false authority can arise even for a statistically ideal text-only predictor. The model is not making an optimization mistake; it is applying the correct marginal law given the information available in text. The error is epistemic: the current latent state is random, but the observed prefix is compatible with a deterministic structural regime. The model cannot distinguish a genuine invariant from a misleading local accident without additional information.

For this reason, the term \emph{sufficiency gap} is mathematically more precise than \emph{ergodicity gap} in the theorem itself. In the broader interpretation of \citet{corielli2026}, heterogeneous non-ergodic archives amplify this problem because they contain many local regimes whose textual signatures can be confused, transferred, or overgeneralized.

\section{Temperature Scaling Cannot Restore Latent Sufficiency}

A common response to overconfident generations is to increase sampling temperature, one of several decoding choices known to affect diversity and degeneration in neural text generation \citet{holtzman2020}. Temperature, however, reshapes the model-induced distribution; it does not restore omitted latent variables. In the binary case, if the unscaled text-only marginal probability of the alternating token is $\alpha>1/2$, temperature $T>0$ produces
\begin{equation}
\label{eq:temperature}
p_T(X_{t+1}=1-X_t\mid\mathcal H_t)
=\frac{\alpha^{1/T}}{\alpha^{1/T}+(1-\alpha)^{1/T}}
=\frac{1}{1+\left(\frac{1-\alpha}{\alpha}\right)^{1/T}}.
\end{equation}

\begin{theorem}[Temperature increases structural error in the deterministic regime]
\label{thm:temperature}
Assume the true continuation regime is Regime 0 and $\alpha\in(1/2,1)$. For every $T>1$, temperature scaling strictly increases the probability of sampling the structurally invalid token relative to $T=1$. Moreover, this change does not reduce the residual dependence of $X_{t+1}$ on the latent state.
\end{theorem}

\begin{proof}
In Regime 0, the valid continuation is $1-X_t$ with probability one. The probability of sampling the invalid token at temperature $T$ is
\begin{equation}
\epsilon(T)=1-p_T(X_{t+1}=1-X_t\mid\mathcal H_t)
=\frac{r^{1/T}}{1+r^{1/T}},
\qquad r:=\frac{1-\alpha}{\alpha}\in(0,1).
\end{equation}
For $r\in(0,1)$, the function $T\mapsto r^{1/T}$ is strictly increasing for $T>0$. Since $u\mapsto u/(1+u)$ is strictly increasing for $u>0$, $\epsilon(T)$ is strictly increasing in $T$. Thus $T>1$ strictly increases the probability of structural error relative to $T=1$.

Temperature acts only on the predictive probabilities in Eq.~\eqref{eq:temperature}. It does not condition on $Z_{t+1}$ and therefore cannot make $X_{t+1}$ independent of $Z_{t+1}$ conditional on the text. It reshapes the marginal distribution but does not restore latent sufficiency.
\end{proof}

Lowering temperature has the opposite local effect: it sharpens the marginal distribution and can reduce errors in deterministic regimes. But this comes at the cost of intensifying false authority when the true regime is random. Temperature therefore moves the system along a confidence-diversity frontier. It does not convert a text-only marginal law into the full situation-specific conditional process.

\section{Contextual Augmentation by a $\gamma$-Oracle}

We now introduce an auxiliary signal $R_t\in\{0,1\}$ representing retrieved evidence, a tool output, a human annotation, a programmatic check, or another external observer; this abstraction covers retrieval-augmented and tool-augmented language-modeling systems \citet{lewis2020,borgeaud2022,karpas2022,schick2023,gao2023,yao2023}. The signal has fidelity $\gamma\in[1/2,1]$:
\begin{equation}
\Prb(R_t=k\mid Z_{t+1}=k)=\gamma,
\qquad k\in\{0,1\}.
\end{equation}
When $\gamma=1/2$, the signal is uninformative noise. When $\gamma=1$, it perfectly reveals the latent regime.

\begin{assumption}[Learned conditional availability]
\label{ass:learned-availability}
The augmented predictor has learned the conditional semantics of the auxiliary signal $R_t$. That is, it can treat $R_t$ as evidence about $Z_{t+1}$ according to the emission model above.
\end{assumption}

This assumption is essential. An external token may be informationally sufficient from the standpoint of an external statistician but useless for a model that has not learned how such tokens constrain continuation. In the terminology of \citet{corielli2026}, informational sufficiency and learned conditional availability are distinct requirements. RAG passages, tool outputs, and prompt-injected facts are useful only when they both contain relevant information and belong to a learned conditioning pattern.

\begin{theorem}[Bayesian updating under imperfect grounding]
\label{thm:gamma-update}
Let the true state be Regime 1 and suppose the observed textual history induces the text-only posterior $\pi_{t,0}=\Prb(Z_{t+1}=0\mid\mathcal H_t)$. If the auxiliary signal emits the corrective value $R_t=1$, then the augmented posterior probability assigned to the deterministic regime is
\begin{equation}
\label{eq:q-gamma}
q_{t,\gamma}
:=\Prb(Z_{t+1}=0\mid\mathcal H_t,R_t=1)
=\frac{(1-\gamma)\pi_{t,0}}
{(1-\gamma)\pi_{t,0}+\gamma(1-\pi_{t,0})}.
\end{equation}
Consequently, the augmented probability assigned to the alternating token is
\begin{equation}
\label{eq:alpha-gamma}
\alpha(\gamma)
=\frac12\left(1+q_{t,\gamma}\right)
=\frac12\left(1+\frac{(1-\gamma)\pi_{t,0}}
{(1-\gamma)\pi_{t,0}+\gamma(1-\pi_{t,0})}\right).
\end{equation}
\end{theorem}

\begin{proof}
By the law of total probability,
\begin{align}
&\Prb(X_{t+1}=1-X_t\mid\mathcal H_t,R_t=1) \\
&\qquad =q_{t,\gamma}\cdot 1+(1-q_{t,\gamma})\cdot\frac12
=\frac12(1+q_{t,\gamma}).
\end{align}
It remains to compute $q_{t,\gamma}$. By Bayes' theorem, using the usual posterior-updating logic of graphical and Bayesian models \citet{wainwright2008},
\begin{align}
q_{t,\gamma}
&=\frac{\Prb(R_t=1\mid\mathcal H_t,Z_{t+1}=0)\Prb(Z_{t+1}=0\mid\mathcal H_t)}
{\sum_{k\in\{0,1\}}\Prb(R_t=1\mid\mathcal H_t,Z_{t+1}=k)\Prb(Z_{t+1}=k\mid\mathcal H_t)}.
\end{align}
Conditional on $Z_{t+1}$, the signal emission is independent of $\mathcal H_t$. Thus
\begin{equation}
\Prb(R_t=1\mid Z_{t+1}=0)=1-\gamma,
\qquad
\Prb(R_t=1\mid Z_{t+1}=1)=\gamma.
\end{equation}
Substitution gives Eq.~\eqref{eq:q-gamma}, and Eq.~\eqref{eq:alpha-gamma} follows immediately.
\end{proof}

\section{The Contextual Dominance Threshold}

Assume now that the alternating textual history induces a dominant misleading prior:
\begin{equation}
\label{eq:dominant-prior}
\pi_{t,0}>\frac12.
\end{equation}
The external signal reverses posterior odds when the posterior assigned to Regime 0 falls below $1/2$:
\begin{equation}
q_{t,\gamma}<\frac12.
\end{equation}
Solving $q_{t,\gamma}=1/2$ using Eq.~\eqref{eq:q-gamma} gives
\begin{equation}
\label{eq:gamma-critical}
\gamma_{\mathrm{crit}}=\pi_{t,0}.
\end{equation}

\begin{corollary}[Contextual dominance threshold]
\label{cor:threshold}
Under Eq.~\eqref{eq:dominant-prior}, the corrective signal $R_t=1$ reverses the posterior odds in favor of the random regime if and only if
\begin{equation}
\gamma>\pi_{t,0}.
\end{equation}
If $\gamma<\pi_{t,0}$, the deterministic textual prior remains posteriorly dominant.
\end{corollary}

\begin{proof}
From Eq.~\eqref{eq:q-gamma}, the inequality $q_{t,\gamma}<1/2$ is equivalent to
\begin{equation}
(1-\gamma)\pi_{t,0}<\gamma(1-\pi_{t,0}),
\end{equation}
which simplifies to $\gamma>\pi_{t,0}$. Equality holds at $\gamma=\pi_{t,0}$.
\end{proof}

The threshold in Corollary~\ref{cor:threshold} is not a thermodynamic phase transition. The posterior $q_{t,\gamma}$, the predictive probability $\alpha(\gamma)$, and the entropy $H_2(\alpha(\gamma))$ vary continuously in $\gamma$ on $(1/2,1]$ except at degenerate boundary cases of zero denominator. The threshold is instead a posterior-odds crossing. It marks the point at which the external signal becomes stronger than the archive-induced structural prior; in repeated-information settings, posterior convergence and merging of beliefs require stronger conditions on the information stream \citet{blackwell1962}.

\begin{corollary}[Gap reduction versus gap closure]
\label{cor:closure}
For $\gamma>\pi_{t,0}$, the corrective signal reduces the sufficiency gap by shifting posterior mass toward the true random regime. However, the gap is fully closed only when $\gamma=1$ or, more generally, when $q_{t,\gamma}=0$.
\end{corollary}

\begin{proof}
When $\gamma>\pi_{t,0}$, Corollary~\ref{cor:threshold} gives $q_{t,\gamma}<1/2$, so
\begin{equation}
\alpha(\gamma)=\frac12(1+q_{t,\gamma})<\frac34.
\end{equation}
This is closer to the true random-regime probability $1/2$ than a posterior-dominant deterministic prediction. However, if $q_{t,\gamma}>0$, then $\alpha(\gamma)>1/2$ and $H_2(\alpha(\gamma))<1$, so the entropy gap remains positive. Eq.~\eqref{eq:q-gamma} yields $q_{t,\gamma}=0$ when $\gamma=1$ and $\pi_{t,0}<1$. In that case $\alpha(1)=1/2$ and the entropy equals one bit, matching the true random-regime conditional entropy.
\end{proof}

This result gives a precise interpretation of retrieval and tool failures and is also consistent with Bayesian interpretations of in-context learning \citet{xie2022}. Adding material to a context window is not enough. The material must be sufficiently reliable relative to the model's text-induced prior, and the model must have learned to treat that material as evidence. Otherwise the learned archive can absorb, discount, or reinterpret the corrective signal.

\section{External Observers, Verification, and Self-Blindness}

The preceding results clarify why an autoregressive sequence model is structurally limited as a self-evaluator. Its internal confidence is computed from the distribution available to it. Under the false-authority event of Theorem~\ref{thm:sufficiency-gap}, the model's internal entropy is
\begin{equation}
H_{\mathrm{internal}}=H(p_{\mathrm{marg}}\mid\mathcal H_t)=H_2\!\left(\frac12(1+\pi_{t,0})\right).
\end{equation}
This value is a function of the text-only posterior. It is not a function of the realized latent state except through whatever information about that state is already encoded in the text.

Thus, the model cannot compute the sufficiency gap from its own text-only distribution alone. The quantity
\begin{equation}
\Delta H
=H(p_{\mathrm{full}}\mid\mathcal H_t,Z_{t+1})
-H(p_{\mathrm{marg}}\mid\mathcal H_t)
\end{equation}
requires access to the latent-state-conditioned law. That information has been integrated out. Softmax confidence, predictive entropy, or logit margins can therefore be calibrated with respect to historical textual patterns while remaining miscalibrated with respect to world-state correspondence, a distinction that connects to the hallucination and uncertainty-estimation literatures \citet{ji2023,kadavath2022}.

This is the formal role of an external observer. The observer may be a human expert, a retrieval system, a database query, a compiler, a theorem prover, a calculator, a laboratory measurement, or another model with access to different evidence; instruction-following and orchestration mechanisms can help route such evidence into the interaction, but they do not by themselves remove the need for verification \citet{ouyang2022}. The crucial property is not that the observer is non-machine, but that it is structurally decoupled from the same text-only marginalization failure.

If an auxiliary architecture can reliably access or verify the relevant latent state, then the sequence model should not be treated as an autonomous epistemic authority. The robust component should constrain, verify, reject, or govern generation. Whether it should replace the sequence model entirely depends on the task. In many applications, language modeling remains useful as an interface, summarizer, or proposal generator, while the observer supplies grounding and verification. In high-stakes settings, however, the decision rule should be organized around the verifier or external state-access mechanism rather than around the unaudited next-token distribution.

\section{Relation to Programming, RAG, and Tool Use}

The binary toy model is intentionally minimal, but it captures the two-stage criterion developed in \citet{corielli2026}. First, the model must identify the relevant local regime from the prefix. Second, within that regime, the prefix and any augmentation must make the latent circumstances approximately irrelevant:
\begin{equation}
I_k(X_{t+1};Z_t^{(k)}\mid X_{\le t},R_t,A_t)\approx 0,
\end{equation}
where $R_t$ denotes retrieved material and $A_t$ denotes tool output.

Programming is a favorable domain because syntax, previous code, tests, error messages, documentation, and execution traces often textualize much of the relevant latent state. This does not make generated code automatically correct, but it makes the sufficiency condition more plausible than in open-world factual discourse. Moreover, compilers, unit tests, type checkers, and runtime execution can act as external observers whose objectives are not merely next-token likelihood.

RAG is useful under the same condition, as retrieval-augmented systems are best understood here as attempts to textualize missing state rather than as automatic truth guarantees \citet{lewis2020,borgeaud2022}. Retrieved material helps when it reduces residual dependence on omitted circumstances. It fails when it merely adds topical text, citation-like surface form, or misleadingly similar context without capturing the latent facts needed for the continuation. Tool use is stronger than retrieval when the tool has direct access to non-textual state or executes a formal procedure, as in program-aided or action-augmented language-modeling systems, but it is still useful only if the tool is appropriate for the latent variable at issue and the model has learned how to condition on the result \citet{gao2023,schick2023,yao2023}.

This distinction is especially important for prompt injection of facts or rules. A prompt can select among learned conditional behaviors; it does not by itself install a new conditional law, which is consistent with evidence that in-context examples often work by activating or selecting learned patterns rather than by unconstrained rule acquisition \citet{brown2020,min2022,wei2022,xie2022}. Even when an added context token is informationally sufficient in principle, the model must have learned how that kind of token constrains continuation. This is exactly the role of Assumption~\ref{ass:learned-availability} in the toy model.

The same sufficiency distinction also matters for recursive training and synthetic contamination. If generated text is later treated as ordinary training evidence, the next model is no longer trained only on samples from the human text process but also on samples from a model-induced distribution. When the generating model operated under an insufficient prefix, the synthetic sample may preserve a plausible marginal continuation rather than the full situation-specific conditional. This mechanism complements recent analyses of self-consuming generative loops and model collapse, where recursively generated data can reduce diversity, erase distributional tails, or move the training target away from the original data-generating process \citet{shumailov2023,shumailov2024,alemohammad2024}.

\section{Decision-Theoretic Implications}

In practice, inferential errors should be evaluated through a loss function. A wrong continuation in a poem, a wrong code suggestion, a wrong medical recommendation, and a wrong legal interpretation do not have the same cost. The problem formalized here is that a text-only sequence model cannot generally estimate the relevant situation-specific probability of error from its internal distribution alone.

Let $L(a,z)$ denote the loss of taking action or emitting answer $a$ in latent state $z$. A rational decision rule requires an estimate of
\begin{equation}
\E[L(a,Z_{t+1})\mid \mathcal H_t,\text{available evidence}].
\end{equation}
If the available evidence is only $\mathcal H_t$, and if $\mathcal H_t$ is not sufficient for $Z_{t+1}$, then the model's token probabilities can be poor proxies for expected loss. External observers are therefore not optional embellishments in high-stakes environments. They are the mechanisms through which the missing state variables, verification procedures, and domain-specific loss assessments enter the decision system.

\section{Conclusion}

This paper provides a formal companion to \citet{corielli2026} by showing that false authority can emerge even under an ideal text-only predictor. The binary mixed-regime process separates statistical optimality relative to a marginal archive from epistemic usefulness in a realized latent state. The model's overconfidence is not caused by insufficient capacity or bad optimization. It follows from conditioning on an insufficient statistic.

Temperature scaling cannot solve this problem because it modifies the decoding distribution without restoring the omitted latent circumstances. External grounding can help, but only under two conditions. First, the grounding signal must be informative enough to overcome the misleading posterior induced by the textual history. In the binary model, this posterior-odds threshold is $\gamma>\pi_{t,0}$. Second, the model must possess learned conditional availability: it must know how to use the grounding signal as evidence. Otherwise, the signal may be present in the prompt but fail to function as a conditioning variable.

The strongest conclusion is not that sequence models are useless. It is that their usefulness is conditional. They are most reliable in local islands where the textual prefix identifies the regime and textualizes the relevant latent state. They are least reliable when the correct continuation depends on omitted facts, unobserved circumstances, or verification procedures not represented in the prompt. In such cases, safe deployment requires structurally decoupled external observers, retrieval systems, tools, formal checks, or human review. Without them, the system does not take a calculated risk; it acts under an unmeasured sufficiency failure.

\bibliographystyle{plainnat}

\end{document}